\documentclass[11pt]{article}

\usepackage[T1]{fontenc}
\usepackage{amsmath,amssymb,amsthm}
\usepackage{booktabs}
\usepackage[a4paper,margin=1in]{geometry}
\usepackage[hidelinks]{hyperref}

\newtheorem{proposition}{Proposition}

\begin{document}

\begin{titlepage}
\thispagestyle{empty}
\begin{center}
{\LARGE\bfseries Cognitive Schemas, Laws and Tasks\par}
\vspace{1.5em}
{\large Antal Jakov\'ac\textsuperscript{1}\quad Andr\'as Telcs\textsuperscript{1,*}\par}
\vspace{1.5em}
\begin{minipage}{0.92\textwidth}
\small
\textsuperscript{1}HUN-REN Wigner Research Centre for Physics, Institute for
Particle and Nuclear Physics, Computational Sciences Department, Data and
Compute Intensive Sciences Research Group, Budapest, Hungary.


\vspace{1.0em}
E-mail: \href{mailto:jakovac.antal@wigner.hun-ren.hu}{jakovac.antal@wigner.hun-ren.hu}

\vspace{0.6em}
\textsuperscript{*}\textbf{Corresponding author:} Andr\'as Telcs,\newline
\href{mailto:telcs.andras@wigner.hun-ren.hu}{telcs.andras@wigner.hun-ren.hu}

\vspace{0.6em}
ORCID (Andr\'as Telcs): 0000-0002-3205-3081
\end{minipage}
\end{center}
\end{titlepage}

\begin{abstract}
This paper asks how explicit representations can support reusable cognitive
schemas in knowledge-based problem solving. We develop a structural framework
in which schemas are organized by the information and relations required for
their use, rather than introduced as unrelated primitives. The framework also
distinguishes context-dependent relations from more stable structures that can
be reused across different representations.

Tasks are described through the information available, the unknowns to be
determined, and the constraints that admissible solutions must satisfy. This
makes it possible to separate limitations of the representation from
limitations of the solving procedure. In particular, we distinguish
inconsistency, underdetermination, and contextual insufficiency, where the
current representation lacks distinctions or relations required by the
external task meaning. We also show formally when a reduction of representation
preserves the task-relevant solution structure.

The resulting task--schema interface offers a structured way to describe
representational conditions relevant to problem solving. It supports the reuse
and stabilization of derived knowledge while remaining independent of the
particular mechanism used to generate candidate solutions. This may provide a
useful component for future solver architectures that combine structured
knowledge, verification, and learned proposal mechanisms.
\end{abstract}


\vspace{1em}

\section{Introduction}

Knowledge representation is a central component of intelligent problem solving: the form in which a task and its objects are represented determines which relations, transformations, and solution procedures are available. This paper develops a structural account of cognitive schemas within the representational framework introduced in \cite{jakovac2025representation} and further discussed in \cite{jakovac2026struct}. That framework provides finite contexts, invariance relations, representative objects, and a Knowledge Space as a substrate for problem solving. Here we ask which higher-level schemas can be derived from this substrate without introducing additional representational primitives.

The main idea is that notions such as equivalence, order, successor, conditional dependence, causal dependence, discrete succession, number-like structure, implication, and logical operations can be represented as derived structural patterns. They are not introduced as independent cognitive faculties or algorithmic modules. Instead, their applicability follows from relations among contexts, representational objects, and admissible transformations.

The central contribution is a dependency-structured hierarchy of such schemas together with a formal task--schema interface. The hierarchy identifies the representational conditions required for progressively richer relational and process structures. A task is represented by a context, given elements, unknown placeholders, and constraints. These jointly determine the admissible instantiations of the unknowns and the schemas that can legitimately be applied. Problem solving can therefore be viewed as constrained instantiation under task-compatible representational structure.

This formulation also makes it possible to distinguish failures of representation from failures of a solving procedure. We identify three basic forms of ill-posedness: inconsistency, underdetermination, and contextual insufficiency. In each case, the task specification fails to provide the representational conditions required for a unique admissible solution.

The contribution lies in the organization of these notions within a common structural framework: contexts are task-relative representational structures, schemas specify conditions of structural applicability, laws capture relations stable across relevant contexts, and task formulation determines which schemas may participate in problem solving. The paper develops this representational layer through the schema hierarchy, the task--schema interface, and explicit conditions for task ill-posedness.

\section{Related Work}

The closest predecessors of the present framework do not form a single
research tradition. Schema theory and cognitive architectures begin from
organized knowledge and action; rough sets and Formal Concept Analysis from
formal distinctions among objects; analogy and program induction from reusable
relational structure; and recent work on language models and knowledge graphs
from the interaction between linguistic competence and explicit knowledge.
These lines address different parts of the representational problem considered
here.

Classical schema theory treats schemas as organized knowledge structures that
guide memory, interpretation, and action
\cite{Bartlett1932,Rumelhart1980}. Related developments in artificial
intelligence introduced frames, scripts, plans, and goals as structured
representations for understanding and problem solving
\cite{Minsky1975,SchankAbelson1977}. Cognitive architectures and
production-system approaches, including ACT-R, Soar, and problem-space
formulations, similarly organize cognition around structured states, operators,
goals, productions, and memory
\cite{NewellSimon1972,AndersonEtAl2004,LairdNewellRosenbloom1987}.
The present work isolates one component of this broader architectural problem: the representational conditions under which particular schemas become applicable to a task.

Several formal representation frameworks are especially close to the present
setting. Rough set theory uses indiscernibility relations and partitions to
study approximation, reducts, and decision rules
\cite{Pawlak1982,Pawlak1991,SkowronRauszer1992}. Formal Concept Analysis
constructs extent--intent concepts and concept lattices from formal contexts
\cite{Wille1982,GanterWille1999}; pattern structures extend this construction
to richer descriptions \cite{GanterKuznetsov2001}, while granular computing
studies granules as units of abstraction and problem solving \cite{Yao2008}.
Conceptual graphs and related knowledge-representation formalisms provide
graph-based representations of concepts and relations
\cite{Sowa1984,BaaderEtAl2003}. 

In the present framework, contexts are task-relative representational partitions, while the Concept Graph serves as a host for concepts, schemas, and stabilized laws. This gives these structures a different role from an FCA concept lattice, a conceptual graph in Sowa's sense, or a domain ontology.

Relational structure is treated in several distinct ways in theories of analogy and abstraction. Structure-mapping theory locates analogical correspondence primarily in relations among objects rather than in matches between isolated attributes \cite{Gentner1983}. Constraint-based accounts add semantic and pragmatic constraints to the structural requirements of a mapping \cite{HolyoakThagard1989}. Conceptual spaces organize concepts through geometric quality dimensions and similarity \cite{Gardenfors2000}, whereas compositional approaches study the preservation of structure under more abstract transformations \cite{PhillipsWilson2010}. These approaches therefore provide different formal views of comparison, similarity, and structural organization. A separate computational question is how reusable structures are learned and combined across tasks; this is addressed more directly in program induction and rule learning.

Program induction and rule learning make the dependence of rule discovery on an available representational language particularly clear. Inductive logic programming induces logical hypotheses from examples and background knowledge \cite{Muggleton1991,CropperEtAl2021}, whereas DreamCoder and related approaches learn reusable abstractions that can be composed across tasks \cite{ellis2023dreamcoder,lake2017building,andreas2017modular}. In each case, the space of candidate rules is delimited by the available primitives and admissible forms of composition. The corresponding representational issue is therefore concrete: a schema or relation can be applied only when the task context supplies the distinctions and relational structure required to express it.

Modern machine learning locates relational structure and invariance at several distinct levels. Relational inductive biases and graph-based models organize computation around relations among represented objects \cite{battaglia2018relational,xu2019how,sun2019rotatE}. Group-equivariant architectures incorporate specified transformation symmetries into the model architecture \cite{cohen2016group}, while invariant risk minimization treats stability across environments as a learning criterion \cite{arjovsky2019invariant}. These examples show that relational structure and invariance may enter through the representation, the architecture, or the learning objective. In the present framework, invariance recognition and representative selection operate on task-relative contexts: they determine which distinctions are preserved under admissible transformations and which may be suppressed. This places the issue at the level of task representation rather than the choice of a particular model class or loss function.

Task specification provides another point of contact. In formal methods and
program synthesis, admissible solutions are defined relative to explicit
specifications and constraints before search is performed
\cite{gulwani2017program,solar2008specification}. Causal models likewise make
variables, structural relations, and admissible interventions explicit
\cite{pearl2009causality}. Work on benchmark design and generalization has also
shown that apparent learning failures may reflect mismatches between data,
representation, and evaluation criteria
\cite{geirhos2020shortcut,schmidt2021characterizing}. The framework likewise makes task structure explicit, while also treating the representational context as part of the task specification. This makes it possible to distinguish inconsistency, underdetermination, and contextual insufficiency from failures of a particular search, learning, or inference procedure.

Recent work on integrating large language models with explicit knowledge
structures is also relevant to the broader motivation of the framework. A
recent TIST special issue identifies three complementary directions: examining
the limits of structured knowledge implicitly available in LLMs, using LLMs to
construct or enrich explicit knowledge structures, and using such structures
to support more grounded retrieval and reasoning
\cite{He2026LLMKG}. Experiments on public ontologies show that
general-purpose pretrained LLMs reproduce exact identifier--label associations
only incompletely, with successful recall strongly associated with concept
frequency on the Web \cite{Bombieri2026Ontologies}. These results caution
against identifying the implicit representation of a pretrained language model
with an explicit and reliably accessible knowledge structure.

Conversely, LLMs can assist the construction and extension of explicit
structures. GPT-assisted entity typing can enrich entity, type-cluster, and
relation representations in an existing knowledge graph
\cite{Zhang2026GMGCL}, while FLAME uses an LLM to propose the placement of new
concepts in an existing expert-curated taxonomy
\cite{Mishra2026FLAME}. Language models can also provide an interface between
external task descriptions and formal structured representations. Meloni
et al.\ study the translation of natural-language scientific questions into
executable SPARQL queries over knowledge graphs and show that explicit entity
and relation information can substantially assist this translation
\cite{Meloni2026SPARQL}. These approaches illustrate different ways in which
the broad linguistic competence of pretrained models can complement explicit
structured knowledge.

Of particular relevance to future solver architectures, Zhang et al.
combine an LLM with an explicit knowledge graph, meta-knowledge, short- and
long-term memory, and a self-correcting mechanism
\cite{Zhang2026MetaKnowledge}. Experience acquired during interaction is
retained and reused across subsequent tasks, providing incremental adaptation
outside the fixed pretrained model. This direction is close in motivation to
the separation considered here between problem-solving mechanisms and an
explicit persistent knowledge structure. The Knowledge Space proposed in our
framework is nevertheless different from these knowledge graphs and
meta-knowledge stores. Its Concept Graph is not primarily a repository of
externally specified entities and factual triples, and its persistent
structures are not limited to strategies for querying an existing graph.
Instead, the Knowledge Space is intended to contain task-relative concepts,
schemas, and stabilized laws that can be constructed and reused across
contexts. The present paper provides the representational layer on which such interpretation, proposal, verification, and controlled knowledge-extension mechanisms can be organized.

Overall, existing work provides rich accounts of structured knowledge,
relations, abstraction, invariance, and formal task specification. The
distinctive feature of the present framework is their organization around an
explicit task-relative context: contexts determine available distinctions,
schemas specify structural applicability, laws capture relational structure
stable across admissible context transformations, and task formulation
determines which of these structures can participate in problem solving.

 \section{Cognitive Schemas and Structural Derivation}

\subsection*{Schemas, Application Frames, and Applicability}

We use cognitive schema as the general term for reusable structural mappings. A cognitive primitive is itself a schema. Primitive schemas belong to the initial cognitive basis, whereas derived schemas use structures made available by other schemas or by previously stabilized conceptual objects. The schemas developed below are therefore presented in a dependency-compatible order: the prerequisites of a derived schema are introduced before the schema that uses them. The dependency structure need not form a single linear chain.

Throughout the paper, $C\subseteq\Omega$ denotes the underlying class or  base domain of a context, $\mathcal C$ denotes a finite context on $C$, and $c\in\mathcal C$ denotes a context cell. When a cell is represented by a  distinguished object, we may use that representative in place of the cell when no ambiguity results. 

A schema $S$ is represented by

\[
S=(I_S,O_S,P_S,\Phi_S),
\]

where $I_S$ specifies its typed input placeholders, $O_S$ its typed output placeholders, $P_S$ its structural prerequisites, and $\Phi_S$ the structural mapping associated with the schema.

The input and output placeholders have the form

\[
I_S=
(x_1:\tau_1,\ldots,x_m:\tau_m),
\]

and

\[
O_S=
(y_1:\sigma_1,\ldots,y_n:\sigma_n),
\]

where $\tau_i$ and $\sigma_j$ are the required types. The types specify which kinds of objects may occupy the corresponding placeholders. They may refer, for example, to contexts, context cells, concepts, relations, laws, schemas, or other structured objects available in the Knowledge Space. No fixed exhaustive collection of types is required here.

The mapping $\Phi_S$ is structural rather than algorithmic. It specifies which output structure becomes available when the input placeholders are instantiated by compatible objects and the prerequisites $P_S$ are satisfied.

A schema is generally not applied to a context alone. Its application also depends on the objects currently available, the outputs sought, and the relevant local constraints. We represent this information by a schema application frame

\[
\mathcal A=(\mathcal C,X,Y,\Gamma_{\mathcal A}),
\]

where $\mathcal C$ is the active context, $X$ is the collection of available objects and structures, $Y$ is the collection of output placeholders available to the application, and $\Gamma_{\mathcal A}$ contains the constraints relevant to that application. Elements of $Y$ may correspond to unknowns of the task or to intermediate objects introduced during problem solving.

An application frame contains more information than the context but need not contain the complete task. For a task

\[
\mathcal T=(\mathcal C,G,U,\Gamma),
\]

an application frame selects the part of the task and of the currently available derived structure that is relevant to one prospective schema application. Different schemas may therefore be applied to different application frames within the same task.

An instantiation of $S$ in $\mathcal{A}$ assigns available objects to the input placeholders of $S$ and associates its output placeholders with compatible elements of $Y$. We write

\[
\operatorname{App}(S,\mathcal{A},\iota)
\]

when an instantiation $\iota$ satisfies the following conditions:

\begin{enumerate}
    \item every input placeholder $x_i:\tau_i$ is assigned an object compatible with the required type $\tau_i$;
    \item every output placeholder $y_j:\sigma_j$ is associated with an output of compatible type $\sigma_j$;
    \item the structural prerequisites $P_S$ are satisfied by the instantiated objects in the context $\mathcal C$;
    \item the instantiation is compatible with the local constraints $\Gamma_{\mathcal{A}}$.
\end{enumerate}

The schema $S$ is applicable to $\mathcal{A}$ if at least one such instantiation exists:

\[
\operatorname{App}(S,\mathcal{A})
\Longleftrightarrow
\exists\iota\;
\operatorname{App}(S,\mathcal{A},\iota).
\]

For an admissible instantiation, $\Phi_S$ determines the corresponding admissible output structure or structures. Applicability therefore depends jointly on type compatibility, the active context, the available objects, and the local constraints.

This formulation also gives a precise meaning to dependency between schemas.
We write

\[
S_1\triangleleft S_2
\]

if a structure produced by an admissible application of $S_1$, or the
stabilized schema concept corresponding to $S_1$, can satisfy an input
placeholder or a structural prerequisite of $S_2$. The relation
$\triangleleft$ denotes direct structural dependency.

Dependency does not by itself imply applicability. If

\[
S_1\triangleleft S_2,
\]

an application of $S_1$ may provide one of the structures required by
$S_2$, but $S_2$ becomes applicable only when its remaining typed inputs,
structural prerequisites, and local constraints are also satisfied in the
corresponding application frame. In this sense, derivation denotes typed
structural construction rather than logical implication between schema
names.

For example,

\[
S_{\mathrm{order}}
\triangleleft
S_{\mathrm{succ}}
\]

expresses that an ordered structure can satisfy a prerequisite of the
successor schema. Once the remaining applicability conditions are met,
the successor schema can use this structure to construct the corresponding
minimal-progression relation.

More generally, a derived schema is a schema whose admissible application
uses structural objects made available by preceding schema applications or
by their stabilized conceptual representations, rather than relying only on
the initial cognitive basis.

Schemas may themselves become objects of representation. When a schema $S$ is sufficiently stable for reuse across a family of applications, it may be lifted into the Concept Graph as a schema concept. We denote this by

\[
\operatorname{Lift}(S)=c_S.
\]

The schema $S$ and the concept $c_S$ have distinct functions. The schema is the structural mapping that can be instantiated and applied, whereas $c_S$ represents that schema as a conceptual object. The lifted object may therefore be referred to, compared, generalized, specialized, or used as an input to another schema.

The same distinction applies when a schema refers to an ordinary concept. A concept node represents a stable conceptual object, whereas a particular schema application may require a context-relative realization of that concept. Let

\[
\operatorname{Inst}_{\mathcal{A}}(c)
\]

denote the admissible realizations of a concept $c$ in the application frame $\mathcal{A}$. A concept may have one or more associated instantiation schemas. If $S_c^{\mathrm{inst}}$ is such a schema, an admissible application has the form

\[
S_c^{\mathrm{inst}}(c,\mathcal{A})=x_{\mathcal{A}},
\]

with

\[
x_{\mathcal{A}}
\in
\operatorname{Inst}_{\mathcal{A}}(c).
\]

Thus, when an input placeholder of a schema requires a context-relative object but is associated with a concept $c$, the corresponding instantiation schema supplies an admissible realization before the main schema is applied. Schematically,

\[
c
\longrightarrow
x_{\mathcal{A}}
\longrightarrow
\Phi_S(x_{\mathcal{A}},\ldots).
\]

This is different from a meta-level application. If an input placeholder explicitly requires a concept or a schema concept, the conceptual object itself may be used as the input. A higher-level schema may therefore operate directly on previously stabilized schemas. For example, if $S_1$ and $S_2$ have been lifted,

\[
\operatorname{Lift}(S_1)=c_{S_1},
\qquad
\operatorname{Lift}(S_2)=c_{S_2},
\]

a higher-level schema $H$ may use these schema concepts as inputs and produce a new schema $S_3$:

\[
(c_{S_1},c_{S_2})
\xrightarrow{H}
S_3.
\]

If $S_3$ subsequently becomes stable for reuse, it may itself be lifted:

\[
S_3
\xrightarrow{\operatorname{Lift}}
c_{S_3}.
\]

Consequently, schema construction is not restricted to a fixed hierarchy. Primitive schemas can support derived schemas; stabilized schemas can become conceptual objects; and these schema concepts can participate in the construction of further schemas. This provides a structural basis for an evolving Knowledge Space while preserving the distinction between a schema, its conceptual representation, its context-relative instantiations, and its executable realizations.

In the present paper, the role of the Knowledge Space is deliberately limited.
Its graph structure provides a persistent organization in which stabilized
schemas and laws can be represented as conceptual objects and made available
for later schema applications. The constructions developed here use this
organization through lifting, instantiation, and reuse of such objects; they do
not depend on a particular graph-search, retrieval, or navigation mechanism.
Those mechanisms belong to the solver level of the framework.

Thus the Knowledge Space is used here as the representational substrate that
allows a derived structure to change status: a schema first appears as an
applicable structural mapping, may become stabilized as a schema concept, and
may then participate as an input to the construction of further schemas. In
schematic form,

\[
S
\longrightarrow
c_S
\longrightarrow
S'
\]

denotes this transition from constructed structure to reusable conceptual
object and subsequently to a component of further structural construction.

\subsection*{Equivalence}

Let $\mathcal{C}$ be a context on a base set $C \subseteq \Omega$, i.e.\ a
partition of $C$ into disjoint cells.
Once $\mathcal{C}$ is fixed, elements $\omega,\omega' \in C$ are considered
\emph{equivalent} if they belong to the same context cell:
\[
\omega \sim_{\mathcal{C}} \omega' \quad \Longleftrightarrow \quad
\exists c \in \mathcal{C} \text{ such that } \omega,\omega' \in c .
\]
Equivalence is therefore induced by the representational granularity encoded in
the context.
Identity corresponds to the degenerate case in which all cells are singletons.
Refining or coarsening a context modifies the induced equivalence relation.

\subsection*{Order and Comparison}

An order schema introduces a strict asymmetric relation $\prec$ defined over context cells or their representatives. Formally,

\[
\prec\;\subseteq\;\mathcal C\times\mathcal C
\]

is a strict partial or strict total order. Order relations operate on equivalence classes rather than on raw states and enable comparison, precedence, containment, or hierarchy without requiring numerical structure.

\subsection*{Successor}

Given an order $\prec$, the successor schema refines it by identifying minimal progression. For $a,b\in\mathcal C$, define

\[
\operatorname{succ}_{\prec}(a,b)
\]

if

\[
a\prec b
\]

and there is no $c\in\mathcal C$ such that

\[
a\prec c\prec b.
\]

The successor relation is defined abstractly and applies to any ordered structure. Temporal order provides a canonical instance, but no commitment to physical or metric time is assumed.

\paragraph{Time and successor structure.}
Successor structure need not be temporal. When a temporal coordinate is available, the successor relation may be interpreted as temporal precedence and thereby supports before--after reasoning. However, the framework does not assume metric, physical, or global time. The same structural role may be played by any ordered progression admitting a successor relation, such as process states, refinement chains, or stages of representation.

\subsection*{Structural Dependence}

Structural dependence describes how the admissible realization of one
component of an application frame constrains the admissible realizations of
another. It is a structural notion defined from the available objects and
constraints. It is distinct from truth-functional implication and does not
assume a probabilistic interpretation.

Let

\[
\mathcal A=(\mathcal C,X,Y,\Gamma_{\mathcal A})
\]

be an application frame. Let $\mathfrak I_{\mathcal A}$ denote the set of
admissible instantiations of this frame, that is, the instantiations that
satisfy the relevant type requirements and the constraints
$\Gamma_{\mathcal A}$.

Let $A$ and $B$ be two typed components whose realizations are specified by
an instantiation $\iota\in\mathfrak I_{\mathcal A}$. For an admissible value
$a$ of $A$, define the admissible realizations of $B$ conditional on
$A=a$ by

\[
\operatorname{Adm}_{\mathcal A}(B\mid A=a)
=
\lbrace
b :
\exists\iota\in\mathfrak I_{\mathcal A}
\text{ such that }
\iota(A)=a
\text{ and }
\iota(B)=b
\rbrace.
\]

We say that $B$ is structurally dependent on $A$ in the application frame
$\mathcal A$, and write

\[
A\leadsto_{\mathcal A}B,
\]

if there exist two admissible realizations $a_1$ and $a_2$ of $A$ such that

\[
\operatorname{Adm}_{\mathcal A}(B\mid A=a_1)
\neq
\operatorname{Adm}_{\mathcal A}(B\mid A=a_2),
\]

with both sets nonempty.

Thus

\[
A\leadsto_{\mathcal A}B
\]

means that changing the admissible realization of $A$ changes the admissible
possibilities for $B$ under the constraints represented in
$\mathcal A$. The dependence is therefore relative to the application
frame: changing the context, the available structures, or the relevant
constraints may change whether the dependence holds.

Structural dependence is directional in the sense that the admissible
possibilities for $B$ are compared under different realizations of $A$.
The relation need not, however, be asymmetric. It is possible that

\[
A\leadsto_{\mathcal A}B
\]

and

\[
B\leadsto_{\mathcal A}A
\]

both hold when the constraints couple the two components in both
directions. Structural dependence alone therefore does not specify causal
direction.

At the schema level, the construction may be summarized as

\[
S_{\mathrm{dep}}:
(\text{application frame},\text{typed components},\text{constraints})
\longrightarrow
\text{structural dependence}.
\]

The resulting dependence relation may subsequently serve as an input or
structural prerequisite of higher-level schemas. In particular, when
successor structure is also available, structural dependence can be refined
by an explicit precedence requirement. This gives the restricted causal
schema developed in the next subsection.
\subsection*{General and Conditional Implication}

Implication is defined separately from structural dependence. Let
$\mathcal C$ be a context on $C\subseteq\Omega$, and let $A$ and $B$ be
predicates on $C$ compatible with the distinctions represented by
$\mathcal C$.

A general implication holds over the complete domain of the context:

\[
A\Rightarrow_{\mathcal C}B
\]

if

\[
\forall \omega\in C,\qquad
A(\omega)\Longrightarrow B(\omega).
\]

A conditional implication holds only on a specified subset $D\subseteq C$.
We write

\[
A\Rightarrow_D B
\]

if

\[
\forall \omega\in D,\qquad
A(\omega)\Longrightarrow B(\omega).
\]

Thus structural dependence,

\[
A\leadsto_{\mathcal A}B,
\]

and implication,

\[
A\Rightarrow_D B,
\]

are different schemas. Structural dependence concerns how admissible
realizations vary under the constraints of an application frame, whereas
implication concerns preservation of truth over a specified domain.

\subsection*{Structural Causal Precedence}

We introduce here one restricted causal schema constructed from two
previously available structures: structural dependence and successor. It
captures a class of causal relations in which a directed dependence is
compatible with an explicit predecessor structure. The construction is not
intended as an exhaustive definition of causality; rather, it provides one
causal schema that can be derived from the structural elements available at
this level of the framework.

Let

\[
\mathcal A=(\mathcal C,X,Y,\Gamma_{\mathcal A})
\]

be an application frame. Let $A$ and $B$ be Boolean-valued typed components represented in $\mathcal A$ whose values are well defined on the cells of $\mathcal C$. They therefore induce predicates

\[
A,B:\mathcal C\to\lbrace\top,\bot\rbrace.
\]

Assume further that

\[
A\leadsto_{\mathcal A}B
\]

holds in the sense of structural dependence defined above, and that $\mathcal C$ is equipped with a successor relation $\operatorname{succ}$.

We write

\[
c\prec_{\operatorname{succ}}^{+}d
\]

when $d$ is reachable from $c$ by one or more successor steps. Thus there
exist context cells

\[
c=c_0,c_1,\ldots,c_k=d,
\qquad
k\geq 1,
\]

such that

\[
\operatorname{succ}(c_i,c_{i+1})
\]

for every $i<k$.

For a predicate or event-like pattern $A$, define its occurrence set by

\[
\operatorname{Occ}_{\mathcal C}(A)
=
\lbrace
c\in\mathcal C : A(c)=\top
\rbrace.
\]

We say that $A$ has structural causal precedence over $B$ in
$\mathcal A$, and write

\[
\operatorname{CP}_{\mathcal A}(A,B),
\]

when:

\begin{enumerate}
    \item $B$ occurs in the represented configuration,
    \[
    \operatorname{Occ}_{\mathcal C}(B)\neq\varnothing;
    \]

    \item every occurrence of $B$ has a preceding occurrence of $A$,
    \[
    \forall d\in\operatorname{Occ}_{\mathcal C}(B),
    \qquad
    \exists c\in\operatorname{Occ}_{\mathcal C}(A)
    \text{ such that }
    c\prec_{\operatorname{succ}}^{+}d.
    \]
\end{enumerate}

The first condition supplies directed structural dependence, whereas the second constrains that dependence by successor structure. The
resulting relation is directional because predecessor and successor
positions are distinguished by the context.

At the schema level, the construction has the form

\[
\text{Structural Dependence}
+
\text{Successor}
\longrightarrow
\text{Structural Causal Precedence}.
\]

Equivalently, in terms of schema dependency,

\[
S_{\mathrm{dep}}
\triangleleft
S_{\mathrm{cp}},
\qquad
S_{\mathrm{succ}}
\triangleleft
S_{\mathrm{cp}}.
\]

The present construction therefore defines one restricted causal relation:
a structurally dependent pattern must precede the pattern that depends on
it along the available successor structure. It does not identify all causal
relations with precedence. Rather, the framework permits different causal
schemas to be constructed when additional representational structure is
available. Later extensions may, for example, introduce interventional or
counterfactual causal schemas \cite{pearl2009causality}, or causal relations defined
for deterministic or stochastic processes.
\subsection*{Discrete Succession}

Discrete succession is obtained by repeated application of the successor relation. A finite succession is a sequence of context cells

\[
(c_0,c_1,\ldots,c_k)
\]

such that

\[
\operatorname{succ}(c_i,c_{i+1})
\]

for every $i<k$. Longer successions are obtained by repeated admissible successor extension.

No numerical magnitude is assumed at this stage; only repeated successor structure is required.

\subsection*{Number-Like Structure}

The sequence from order to number-like structure provides a simple example of how derived schemas use the outputs of their predecessors.

Let $\mathcal{C}$ be a context in which an admissible order relation $\prec$ is available. The successor schema can then be applied to the ordered structure,

\[
(\mathcal{C},\prec)
\longmapsto
\operatorname{succ}_{\prec},
\]

where $\operatorname{succ}_{\prec}$ identifies minimal progression with respect to $\prec$.

The resulting successor relation satisfies an input requirement of the discrete-succession schema. Given an initial element $c_0$, this schema produces finite successor chains

\[
(c_0,c_1,\ldots,c_k),
\]

with

\[
\operatorname{succ}_{\prec}(c_i,c_{i+1})
\]

for every $i<k$.

At the next level, the particular elements occurring in such a chain can be suppressed while preserving its finite successor structure. Successions with the same successor structure are treated as equivalent, and the resulting equivalence classes provide number-like objects. Numerical sameness is thereby associated with sameness of finite succession structure rather than with the identity of the elements realizing the succession.

The construction can be summarized at the schema level as

\[
S_{\mathrm{order}}
\triangleleft
S_{\mathrm{succ}}
\triangleleft
S_{\mathrm{disc}}
\triangleleft
S_{\mathrm{num}}.
\]

Each occurrence of $\triangleleft$ represents a typed structural dependency: the structure produced at one stage supplies an input or prerequisite required by the next schema. The last step also illustrates abstraction, since the resulting object no longer depends on the particular elements used to realize the succession.

If such a number-like schema becomes stable and reusable, it may itself be lifted as a schema concept and become available as an input to further schemas. Thus the construction is not intended as an axiomatization of the natural numbers, but as an example of how progressively higher-level structures can be built and subsequently reused within the Knowledge Space.

\subsection*{Higher-Order Structural Correspondence}

The schemas considered so far construct relations and structures within a context. A further step becomes possible when previously available relations or schema concepts are themselves used to identify a common organization across different configurations. The resulting schema operates not only on individual objects but on relational structures formed from them.

Consider two finite collections of typed objects

\[
X=\lbrace a,b,c,d\rbrace,
\qquad
X'=\lbrace r,p,s,q\rbrace.
\]

The order in which the elements of these sets are displayed has no structural meaning. Suppose that the first configuration contains a binary relation $R$ with

\[
R(a,b),\qquad
R(a,c),\qquad
R(b,d),\qquad
R(c,d),
\]

while the second contains a possibly different binary relation $R'$ with

\[
R'(p,q),\qquad
R'(p,r),\qquad
R'(q,s),\qquad
R'(r,s).
\]

The two configurations have the same relational organization under the correspondence

\[
\varphi:X\to X'
\]

defined by

\[
\varphi(a)=p,\qquad
\varphi(b)=q,\qquad
\varphi(c)=r,\qquad
\varphi(d)=s.
\]

For the participating pairs,

\[
R(x,y)
\Longleftrightarrow
R'(\varphi(x),\varphi(y)).
\]

The correspondence is therefore determined by preservation of relational structure rather than by the names or displayed order of the objects.

This construction can be stated more generally. Let $X$ and $X'$ be finite collections of typed objects, and let $\mathcal R$ and $\mathcal R'$ be collections of relations represented in the corresponding application frames. A structural correspondence consists of an object mapping

\[
\varphi:X\to X'
\]

together with, when the participating relations are not identical, a correspondence

\[
\psi:\mathcal R\to\mathcal R'
\]

between the relevant relations. For each participating relation $R\in\mathcal R$, structural preservation requires

\[
R(x_1,\ldots,x_k)
\Longleftrightarrow
\psi(R)
\bigl(
\varphi(x_1),\ldots,\varphi(x_k)
\bigr)
\]

for the tuples involved in the correspondence.

The mappings $\varphi$ and $\psi$ are subject to the type requirements and
structural prerequisites of the application frame; they are not induced by an
enumeration of the participating objects or relations. The structural
correspondence schema may therefore be represented as

\[
S_{\mathrm{corr}}:
(\text{typed objects},\text{typed relations})
\longrightarrow
\text{structural correspondence}.
\]

The construction becomes higher-order when the participating relations or
structural patterns are outputs of previously available schemas, or when
stabilized schemas have already been lifted as schema concepts. Several such
concepts may jointly support the construction of a new higher-order schema,

\[
(c_{S_1},c_{S_2},\ldots,c_{S_m})
\longrightarrow
S_H.
\]

If the resulting schema becomes stable for reuse, it may itself be lifted,

\[
S_H
\xrightarrow{\operatorname{Lift}}
c_{S_H}.
\]

The correspondence therefore preserves an organization among typed objects
and relations without requiring the participating objects or relations to be
identical across the two configurations.

The constructions above illustrate two forms of schema derivation.
Compositional derivation occurs when the output of one schema supplies an
input or prerequisite of another, whereas higher-order construction may
combine several previously available structures. Table~\ref{tab:schema_dependencies}
summarizes the principal dependencies developed in this section.

\begin{table}[t]
\centering
\small
\caption{Structural prerequisites and resulting structures for the principal
schemas developed in this section. The rows summarize dependencies and do
not represent a single linear sequence of schema construction.
\vspace{1em}}
\label{tab:schema_dependencies}
\begin{tabular}{p{0.20\linewidth} p{0.33\linewidth} p{0.33\linewidth}}
\toprule
Schema & Required structure & Resulting structure \\
\midrule

Equivalence
& Context partition
& Equivalence relation induced by common cell membership \\

Order
& Typed objects or context cells together with an admissible comparison relation
& Ordered relational structure \\

Successor
& Ordered structure
& Minimal progression relation \\

Structural dependence
& Application frame, typed components, admissible instantiations, and local constraints
& Directed dependence of admissible realizations within the application frame \\

General and conditional implication
& Predicates together with a specified domain of validity& Truth-functional implication over the corresponding domain\\

Structural causal precedence
& Structural dependence together with successor structure and cell-wise occurrence predicates& Restricted causal relation combining structural dependence with predecessor structure\\

Discrete succession
& Successor relation and an initial object
& Finite or iterated successor chain \\

Number-like structure
& Finite successor chains and equivalence under preservation of successor structure
& Equivalence classes of finite successions and their induced structural organization \\

Higher-order structural correspondence
& Two typed relational configurations together with compatible objects, relations, or previously derived structures
& Structure-preserving correspondence and a reusable higher-order structural pattern \\

\bottomrule
\end{tabular}
\end{table}

One compositional dependency chain is

\[
S_{\mathrm{order}}
\triangleleft
S_{\mathrm{succ}}
\triangleleft
S_{\mathrm{disc}}
\triangleleft
S_{\mathrm{num}}.
\]

Higher-order structural correspondence instead combines several available
structures and therefore need not have a single predecessor.

The constructions considered so far concern the derivation and application of
schemas within contexts or application frames. We next consider when a
relation can be reused across different contexts and thereby represented as a
stable law.

\section{Relations and Laws}

Relations are structured links defined relative to a context. Let
$\mathcal C$ be a context on a base domain $C\subseteq\Omega$. A
context-relative binary relation has the form

\[
R_{\mathcal C}
\subseteq
\mathcal C\times\mathcal C,
\]

or equivalently

\[
R_{\mathcal C}:
\mathcal C\times\mathcal C
\to
\lbrace\top,\bot\rbrace.
\]

Such a relation is evaluated with respect to the distinctions represented by
$\mathcal C$. It may therefore depend on the granularity and organization of
that context and, by itself, is not yet a law.

A relation defined on the underlying domain may induce a context-relative
relation when it is invariant with respect to the context cells. Let

\[
q_{\mathcal C}:C\to\mathcal C
\]

be the canonical quotient map. Let

\[
R:C\times C
\to
\lbrace\top,\bot\rbrace.
\]

The relation $R$ descends to $\mathcal C$ if it is constant on the fibers of

\[
q_{\mathcal C}\times q_{\mathcal C}.
\]

Explicitly, whenever

\[
q_{\mathcal C}(\omega_1)
=
q_{\mathcal C}(\omega'_1)
\]

and

\[
q_{\mathcal C}(\omega_2)
=
q_{\mathcal C}(\omega'_2),
\]

we require

\[
R(\omega_1,\omega_2)
=
R(\omega'_1,\omega'_2).
\]

Under this condition there exists a unique relation $R_{\mathcal C}$ such
that

\[
R
=
R_{\mathcal C}
\circ
(q_{\mathcal C}\times q_{\mathcal C}).
\]

Representative selection provides an equivalent realization of this
context-relative relation. Let

\[
\rho:\mathcal C\to C
\]

satisfy

\[
q_{\mathcal C}\circ\rho
=
\operatorname{id}_{\mathcal C}.
\]

Then, for every $c_1,c_2\in\mathcal C$,

\[
R_{\mathcal C}(c_1,c_2)
=
R\bigl(\rho(c_1),\rho(c_2)\bigr).
\]

Because $R$ is constant on context cells, this value is independent of the
particular representative selection $\rho$. Representative independence is
therefore a well-definedness condition for a context-relative realization,
rather than a transformation between contexts.

We next consider relations across different contexts. Let $\mathcal C'$ be
a context on $C'$ and $\mathcal C$ a context on $C$. A transformation

\[
p:C'\to C
\]

is compatible with the contexts if it induces a map

\[
\pi:\mathcal C'\to\mathcal C
\]

such that

\[
q_{\mathcal C}\circ p
=
\pi\circ q_{\mathcal C'}.
\]

Thus objects that are indistinguishable in $\mathcal C'$ are mapped to
objects that remain indistinguishable in $\mathcal C$. Only transformations
satisfying this compatibility condition are treated here as admissible
context transformations.

Given a relation $R_{\mathcal C}$ on $\mathcal C$, its pullback along
$\pi$ is the relation on $\mathcal C'$ defined by

\[
(\pi^*R_{\mathcal C})(c'_1,c'_2)
=
R_{\mathcal C}
\bigl(
\pi(c'_1),\pi(c'_2)
\bigr).
\]

For a refinement, $\mathcal C'$ is finer than $\mathcal C$ and $\pi$ is the
canonical map sending each finer cell to the coarser cell containing it.
The pullback therefore gives a canonical realization of a coarser relation
on the finer context.

The converse direction requires an additional condition. A relation
$R_{\mathcal C'}$ on the finer context descends to $\mathcal C$ only if it
is constant on the fibers of

\[
\pi\times\pi.
\]

Thus, whenever

\[
\pi(c'_1)=\pi(d'_1)
\]

and

\[
\pi(c'_2)=\pi(d'_2),
\]

we require

\[
R_{\mathcal C'}(c'_1,c'_2)
=
R_{\mathcal C'}(d'_1,d'_2).
\]

When this condition holds, there exists a unique relation
$R_{\mathcal C}$ satisfying

\[
R_{\mathcal C'}
=
\pi^*R_{\mathcal C}.
\]

Hence refinement naturally supports pullback, whereas coarsening requires
descent. A projection between represented domains is admissible only when
it induces a well-defined map between the corresponding context cells as
specified above.

Let $\mathfrak C$ be a family of compatible contexts together with their
admissible context transformations. A law $L$ over $\mathfrak C$ is a
family of context-relative realizations

\[
L=
\bigl(
R^L_{\mathcal C}
\bigr)_{\mathcal C\in\mathfrak C}
\]

such that, for every admissible map

\[
\pi:\mathcal C'\to\mathcal C,
\]

the realizations satisfy

\[
R^L_{\mathcal C'}
=
\pi^*R^L_{\mathcal C}.
\]

The same compatibility equation has two complementary interpretations.
From a coarser context to a finer one, it specifies pullback. From a finer
context to a coarser one, it can hold only when the finer realization
satisfies the required descent condition.

The compatibility condition is itself stable under composition. Suppose

\[
\mathcal C''
\xrightarrow{\sigma}
\mathcal C'
\xrightarrow{\pi}
\mathcal C
\]

are admissible context transformations and

\[
R^L_{\mathcal C'}
=
\pi^*R^L_{\mathcal C},
\qquad
R^L_{\mathcal C''}
=
\sigma^*R^L_{\mathcal C'}.
\]

Then

\[
R^L_{\mathcal C''}
=
\sigma^*
\bigl(
\pi^*R^L_{\mathcal C}
\bigr)
=
(\pi\circ\sigma)^*
R^L_{\mathcal C}.
\]

Thus the realization of a law does not depend on whether a compatible
change of context is performed directly or through intermediate admissible
contexts.

A law is therefore not identified with any single relation
$R^L_{\mathcal C}$. It is the stable relational structure represented by
the compatible family of its context-relative realizations. The family
$\mathfrak C$ specifies its domain of applicability.

When such a law is sufficiently stable for reuse, it may be lifted into the
Concept Graph as a conceptual object. Its context-relative relations remain
realizations of that law rather than independent lifted relations. A lifted
law with explicitly relational content may be represented as a relation
concept.

\section{Tasks, Contexts, and Schema Applicability}

Before problem solving can take place, a task must be represented in a context in which its unknowns, admissible distinctions, and constraints can be stated. A task is therefore specified relative to a context rather than identified with an algorithmic objective, loss function, or input sentence.

Let $C$ be a finite context. A task is represented by

$$
T=(C,G,U,\Gamma),
$$

where $G$ contains the given elements or facts, $U$ the unknowns to be instantiated, and $\Gamma$ the constraints on admissible instantiations.

The external presentation of a task may be verbal, symbolic, diagrammatic, data-based, or already formal. Here we assume that this presentation has been converted into $T$. The questions considered below are then which distinctions and schemas are available in $C$, and whether the unknowns can be instantiated consistently with $\Gamma$.
 \subsection{Illustrative Examples}

The following examples illustrate the consequences of applying an appropriate
representational transformation once the relevant invariance, dependency, or
structural requirement is available. They do not address how such a
transformation is discovered. In some cases the external task semantics
indicates the relevant invariance without prescribing its representational
realization; in others, it specifies that a relevant structure exists without
identifying it completely. The general claims concerning schema applicability,
structural dependency, and cross-context stability follow from the formal
constructions developed above rather than from these examples.    

\subsubsection*{Example 1: Permutation-Invariant Structure}

\paragraph{External task statement.}
Given a finite sequence of symbols, determine whether the sequence contains at least two identical elements. The order of symbols is irrelevant to the answer.

\paragraph{Task specification.}
Let $C_1$ be the domain of finite sequences

\[
x=(x_1,\ldots,x_n)
\]

over a finite alphabet. Let $\mathcal C_1$ be the discrete context on $C_1$, so that each sequence is initially represented by a distinct context cell. The task

\[
\mathcal{T}_1=(\mathcal C_1,G_1,U_1,\Gamma_1)
\]

is defined by:
\begin{itemize}
    \item 
 
\item \textbf{Context $\mathcal C_1$:} the discrete context on $C_1$, in which every sequence is represented by a distinct context cell and the sequence representation retains positional order and symbol identity.

  \item \textbf{Given elements $G_1$:} a finite set of labeled examples
  \[
  \{(x^{(i)},y^{(i)})\},
  \]
  where $x^{(i)}\in C_1$ and $y^{(i)}\in\{0,1\}$.

  \item \textbf{Unknown $U_1$:} a function
  \[
  f:C_1\to\{0,1\}.
  \]

  \item \textbf{Constraints $\Gamma_1$:} consistency with the given examples and invariance under permutations of sequence positions.
\end{itemize}

In the raw context $\mathcal C_1$, two distinct sequences are represented by distinct singleton cells, even when they differ only by a permutation of positions. Define an equivalence relation on $C_1$ by

\[
x\sim_{\mathrm{perm}}x'
\]

if $x'$ can be obtained from $x$ by a permutation of positions. The corresponding quotient context

\[
\mathcal C_1^{\mathrm{perm}}
=
C_1/{\sim_{\mathrm{perm}}}
\]

groups permutation-equivalent sequences into the same context cell and therefore suppresses order distinctions that are irrelevant to the task.

Let

\[
q_{\mathrm{perm}}:
C_1\to\mathcal C_1^{\mathrm{perm}}
\]

denote the canonical quotient map. Permutation invariance means that $f$ is constant on the cells of $\mathcal C_1^{\mathrm{perm}}$. Hence there exists a unique function

\[
\bar f:
\mathcal C_1^{\mathrm{perm}}
\to
\{0,1\}
\]

such that

\[
f=\bar f\circ q_{\mathrm{perm}}.
\]

\paragraph{Data-based variant.}
The same task may be given through labeled observations such as

\[
G_1=
\lbrace
((a,b,c),0),
((a,b,a),1),
((d,e,f),0)
\rbrace.
\]

These observations constrain $f$ only on the displayed sequences. The
additional condition in $\Gamma_1$ specifies permutation invariance, so after
passing to $\mathcal C_1^{\mathrm{perm}}$ the task is represented by

\[
\bar f:
\mathcal C_1^{\mathrm{perm}}
\to
\lbrace0,1\rbrace,
\]

which depends only on the permutation class. In this example, the relevant
class property is whether some symbol occurs at least twice. 

\subsubsection*{Example 2: Abstraction and Generalization by Removing Irrelevant Coordinates}

\paragraph{External task statement.}
Given a binary vector of fixed length, determine an output bit that depends only on a small, unknown subset of input coordinates. All remaining coordinates are irrelevant to the correct answer.

\paragraph{Task specification.}
The task $\mathcal{T}_2=(\mathcal C_2,G_2,U_2,\Gamma_2)$ is defined by:
\begin{itemize}
  \item \textbf{Context $\mathcal C_2$:} binary vectors
  \[
    x=(x_1,\dots,x_d)\in\{0,1\}^d,
  \]
  with all coordinates treated symmetrically.
  \item \textbf{Given elements $G_2$:} labeled examples $\{(x^{(i)},y^{(i)})\}$.
  \item \textbf{Unknown $U_2$:} a function
  \[
    f:\mathcal C_2\to\{0,1\}.
  \]
  \item \textbf{Constraints $\Gamma_2$:} consistency with the given labels.
\end{itemize}

Within the raw context $\mathcal C_2$, all coordinates are equally represented,
so the representation does not identify which coordinates may be ignored
without affecting the output. The task states that a relevant subset exists
but does not specify its members; discovering that subset is outside the
present example.

Once the relevant coordinates have been identified, projection suppresses
distinctions associated with the remaining coordinates. In the reduced
context, admissible laws can be expressed using only the relevant coordinates
and may therefore have simpler representations. This realizes abstraction by
removing task-irrelevant distinctions and supports generalization when the
same reduced description applies across related task instances.

\subsubsection*{Example 3: Relational Composition under Contextual Constraints}

\paragraph{External task statement.}
Given three objects $A,B,C$ and the statements \emph{``$B$ is next to $A$''} and \emph{``$C$ is next to $B$''}, determine whether \emph{``$A$ is next to $C$''} holds.

The task refers to objects and a binary relation \emph{next to}. No metric distance, global geometry, or transitivity assumptions are specified. The meaning of \emph{next to} is restricted to the task context.

\paragraph{Task specification.}
The task $\mathcal{T}_3=(\mathcal C_3,G_3,U_3,\Gamma_3)$ is specified as follows.

The context consists of finite configurations of objects equipped with a binary relation $\mathrm{Next}(\cdot,\cdot)$. The context distinguishes adjacency relations but does not encode distance, order, or spatial embedding.

The given elements are the relational facts
\[
\mathrm{Next}(B,A),\qquad \mathrm{Next}(C,B).
\]

The unknown is the truth value of the relation
\[
\mathrm{Next}(A,C).
\]

\paragraph{Constraints $\Gamma_3$.}
No constraints beyond consistency with the given relations are imposed. In particular, no assumptions about transitivity or symmetry of $\mathrm{Next}$ are included in the task specification.

Within the given context, adjacency is treated as a local relation. The task
context does not determine whether adjacency is transitive or symmetric, and
therefore does not uniquely determine the value of

\[
\mathrm{Next}(A,C).
\]

The original task is consequently underdetermined. Introducing relation
composition or path structure does not by itself resolve this adjacency
question. Such an extension may instead define a different relation, for
example indirect connectivity or reachability, whose truth conditions are not
the same as those of $\mathrm{Next}$.

 A reformulated task could instead ask whether $A$ and $C$ are connected
through intermediate objects. Such a question may become decidable once an
appropriate path or composition schema is available, but the conclusion then
concerns a new relation rather than the original adjacency predicate.
Consequently, adding constraints may resolve underdetermination of
$\mathrm{Next}(A,C)$, whereas introducing a new relational schema constitutes
a reformulation of the task.

\subsection{Solution preservation under representational reduction.}

Let

\[
\mathcal T'
=
(\mathcal C',G',U',\Gamma')
\]

be a task represented in a context $\mathcal C'$, and let

\[
\pi:\mathcal C'\to\mathcal C
\]

be a surjective admissible context transformation that suppresses distinctions
not required by the task. Let $\sim_{\pi}$ denote the corresponding
equivalence on admissible assignments: two assignments are equivalent when
they differ only in distinctions identified by $\pi$.

Assume that satisfaction of the task constraints is invariant under this
equivalence, so that

\[
u'_1\sim_{\pi}u'_2
\quad\Longrightarrow\quad
\left(
u'_1\models_{\mathcal C'}\Gamma'
\Longleftrightarrow
u'_2\models_{\mathcal C'}\Gamma'
\right).
\]

The constraints therefore induce a well-defined task

\[
\mathcal T
=
(\mathcal C,G,U,\Gamma)
\]

on the quotient representation.

\begin{proposition}

Under the conditions above, the solutions of the reduced task are in
one-to-one correspondence with the equivalence classes of solutions of the
original task:

\[
\operatorname{Sol}(\mathcal T)
\cong
\operatorname{Sol}(\mathcal T')/{\sim_{\pi}}.
\]
\end{proposition}

\textit{Proof.}
Constraint invariance implies that either every assignment in a
$\sim_{\pi}$-class satisfies $\Gamma'$ or none does. Satisfaction therefore
depends only on the equivalence class. Each solution class of
$\mathcal T'$ consequently determines exactly one solution of the quotient
task $\mathcal T$, and every solution of $\mathcal T$ is obtained from such
a class.
\qed

Two consequences are immediate. First, a task with no admissible solution
cannot acquire one merely by suppressing distinctions in a
solution-preserving representation. Second, several solutions that differ
only in distinctions suppressed by $\pi$ become a single solution in the
reduced context, whereas solutions belonging to different
$\sim_{\pi}$-classes remain distinct. Representational reduction can
therefore remove apparent underdetermination caused solely by irrelevant
distinctions without changing the task-relevant solution structure.

\section{Ill-Posedness of Tasks}

The task representation makes it possible to distinguish different forms of
ill-posedness at the level of task formulation. Internal inconsistency and
underdetermination are properties of the formalized task, whereas contextual
insufficiency additionally concerns the adequacy of that representation
relative to the external task meaning.

\textbf{Type I: Internal inconsistency.}
A task is internally inconsistent if no admissible instantiation of the unknowns satisfies the constraints in the given context. Formally, there is no assignment $u\in U$ such that

$$
u\models_C\Gamma.
$$

Inconsistency may arise from contradictory constraints, incompatible givens, or a mismatch between the assumed schema and the available context.

\textbf{Type II: Underdetermination.}
A task is underdetermined if several mutually incompatible instantiations of the unknowns satisfy the constraints in the same context. Thus there exist distinct admissible assignments

$$
u_1,u_2\in U
$$

such that

$$
u_1\models_C\Gamma,
\qquad
u_2\models_C\Gamma,
$$

while $u_1$ and $u_2$ are not equivalent with respect to the distinctions relevant to the task. Resolving underdetermination requires additional constraints, additional data, a more specific context, or an evaluation criterion.

\textbf{Type III: Contextual insufficiency.}
A task representation is contextually insufficient when the external task
semantics requires distinctions, relations, schemas, or laws that cannot be
expressed in the current context $\mathcal C$. In this case, the external task
may be meaningful, but its current formal representation

\[
\mathcal T=(\mathcal C,G,U,\Gamma)
\]

does not contain the representational structure required to formulate the
intended question or its admissible solution adequately.

Examples include tasks that require metric distance when only adjacency is
represented, transitivity when only a local relation is available, or temporal
order when no successor structure is represented.

Unlike inconsistency and underdetermination, contextual insufficiency is not
diagnosed from the formal task tuple alone. It is identified relative to the
external task meaning from which the formal representation was constructed.
A more explicit treatment in terms of richer reference contexts and
representational projections can be introduced in later work.

Resolving contextual insufficiency requires a change of representation, for
example by refinement, coarsening, projection, extension, or reformulation,
depending on the missing structure.

Granularity mismatch, scope ambiguity, and schema mismatch can be treated as
refinements of the three cases above. Internal inconsistency and
underdetermination can be diagnosed from the formal task representation,
whereas contextual insufficiency additionally requires comparison with the
external task semantics. Once both the task representation and its relation
to the external task meaning are explicit, these mismatches can be classified
as inconsistency, underdetermination, or contextual insufficiency.

The classification identifies whether failure lies in the constraints $\Gamma$, in the available unknowns $U$, in schema applicability, or in the expressive capacity of the context $C$. 

\subsection*{Objects Lifted into the Concept Graph}

The lifting criterion is stability across contexts. Concepts, schema concepts,
laws, and recurring diagnostic structures may be represented as nodes in the
Concept Graph when they remain stable across an appropriate family of
contexts. Context-specific realizations, including particular relations,
equivalences, and concept instantiations, remain local and are linked to the
corresponding stable objects rather than lifted independently. Families of
compatible contexts may themselves provide reusable domains of applicability
for laws and schemas. Procedures and algorithmic primitives are represented
separately in the Procedure Graph and linked to Concept Graph objects through
typing and applicability relations.

\section{Discussion}

The results above place task formulation and schema applicability at the same representational level. A context fixes which distinctions are available, while the task specification determines which of these distinctions and relations are required by the givens, unknowns, and constraints. Problem solving therefore depends not only on the available procedures, but also on what can be expressed in the current context.

This distinction clarifies several different sources of difficulty. Inconsistency arises when no admissible instantiation satisfies the constraints, underdetermination when several non-equivalent instantiations remain admissible, and contextual insufficiency when the distinctions or relations required by the intended solution are absent from the context. The examples exhibit these cases in elementary form: quotienting removes irrelevant order information, projection suppresses irrelevant coordinates, and insufficient relational structure leaves a query unresolved. A change of representation can therefore alter the admissible solution space without changing the task data themselves.

The distinction between relations and laws extends the same idea across contexts. A relation may depend on the granularity and organization of a particular context. A law has compatible realizations across a family of contexts under admissible transformations such as refinement, coarsening, projection, or representative selection. This stability allows a law to be represented as a reusable object in the Concept Graph, while its context-relative realizations remain tied to particular tasks or contexts.

The schema hierarchy, task--schema interface, and distinction between relations and laws together define a structural layer for problem solving. Within this layer, quotienting, projection, context transformation, and schema applicability can be analyzed independently of the particular mechanism used to discover or select them. This separation also provides a basis for later solver architectures in which proposal generation, verification, and persistent knowledge can be implemented by different components.
    
\section{Conclusion}

This paper developed a hierarchy of cognitive schemas for knowledge-based problem solving. Equivalence, order, successor, dependence, implication, discrete succession, and number-like structure are derived from contexts, relations, representative objects, and admissible transformations rather than introduced as independent primitives.

Task specification connects these schemas to problem formulation. A task context determines which distinctions and relations are available for instantiating the unknowns under the stated constraints. This also separates three forms of ill-posedness: inconsistency, underdetermination, and contextual insufficiency. In the last case, the required distinctions or relations are absent from the current context, so the representation must be changed before a solution can be expressed.

Relations and laws are separated in the same way across contexts. Relations may be specific to a particular context, whereas laws have compatible realizations across a family of contexts under admissible transformations. Such stable structures can be reused and represented in the Concept Graph.

A broader research direction is to investigate how this explicit structural layer can interact with language-model architectures. Two complementary possibilities are of particular interest: building problem-solving capabilities from an initially untrained or minimally pretrained model whose persistent knowledge is organized through the explicit Knowledge Space, and coupling the framework to a pretrained LLM that provides rich natural-language interpretation and candidate proposals while explicit representation, verification, and knowledge stabilization remain outside the model. In both cases, the objective is to separate proposal generation from verified persistent knowledge and controlled structural learning. These directions provide a route for extending the present representational framework toward solver architectures that combine flexible proposal generation with explicit verification and persistent structural learning.

\section{Declarations}

\subsection{Availability of Data and Material}

Data availability statement: not applicable.

\subsection{Competing Interests}

The authors declare that they have no competing interests.

\subsection{Funding}

No funding was received for this work.

\subsection{Authors' Contributions}

The authors contributed equally to all aspects of the work.

\subsection{Acknowledgements}

The authors have no acknowledgements to declare.

\bibliographystyle{unsrt}

\bibliography{ai3_P2}
\end{document}